\documentclass[conference,letterpaper]{IEEEtran}
\IEEEoverridecommandlockouts
\usepackage[letterpaper, left=0.75in, right=0.75in, bottom=0.75in, top=0.75in]{geometry}
\usepackage{cite}
\usepackage{amsmath,amssymb,amsfonts}
\usepackage{algorithmic}
\usepackage{graphicx}
\usepackage{textcomp}
\usepackage{xcolor}

\usepackage{amsmath,amsfonts,bm}

\def\eqref#1{equation~\ref{#1}}

\def\1{\bm{1}}

\DeclareMathAlphabet{\mathsfit}{\encodingdefault}{\sfdefault}{m}{sl}
\SetMathAlphabet{\mathsfit}{bold}{\encodingdefault}{\sfdefault}{bx}{n}

\usepackage{hyperref}
\usepackage{url}
\usepackage[utf8]{inputenc}
\usepackage[T1]{fontenc}
\usepackage{booktabs}
\usepackage{amsfonts}
\usepackage{nicefrac}
\usepackage{microtype}
\usepackage{amsmath}
\usepackage{amsthm}
\usepackage{array}
\usepackage{bm}
\usepackage{graphicx}
\usepackage{thmtools}

\newcommand{\diag}{\operatorname{diag}}

\newcommand{\modelacro}{HMRNN}

\newcommand{\relu}{\operatorname{ReLu}}

\newtheorem{lemma}{Lemma}[section]
\newtheorem{theorem}{Theorem}[section]
\newtheorem{definition}{Definition}[section]

\def\BibTeX{{\rm B\kern-.05em{\sc i\kern-.025em b}\kern-.08em
    T\kern-.1667em\lower.7ex\hbox{E}\kern-.125emX}}
\begin{document}

\title{Hidden Markov models as recurrent neural networks: an application to Alzheimer's disease}

\author{\IEEEauthorblockN{1\textsuperscript{st} Matt Baucum}
\IEEEauthorblockA{\textit{Industrial \& Systems Engineering}\\
\textit{University of Tennessee}\\
Knoxville, TN, U.S.A.\\
mbaucum1@vols.utk.edu}
\and
\IEEEauthorblockN{2\textsuperscript{nd} Anahita Khojandi}
\IEEEauthorblockA{\textit{Industrial \& Systems Engineering} \\
\textit{University of Tennessee}\\
Knoxville, TN, U.S.A. \\
khojandi@utk.edu}
\and
\IEEEauthorblockN{3\textsuperscript{rd} Theodore Papamarkou}
\IEEEauthorblockA{\textit{Department of Mathematics} \\
\textit{The University of Manchester}\\
Manchester, U.K.\\
\textit{Computational Sciences \& Engineering Division} \\
\textit{Oak Ridge National Laboratory}\\
Oak Ridge, TN, U.S.A.\\
theodore.papamarkou@manchester.ac.uk}
}

\maketitle

\begin{abstract}
Hidden Markov models (HMMs) are commonly used for disease progression modeling when the true patient health state is not fully known. Since HMMs typically have multiple local optima, incorporating additional patient covariates can improve parameter estimation and predictive performance. To allow for this, we develop hidden Markov recurrent neural networks ({\modelacro}s), a special case of recurrent neural networks that combine neural networks' flexibility with HMMs' interpretability. The {\modelacro} can be reduced to a standard HMM, with an identical likelihood function and parameter interpretations, but it can also combine an HMM with other predictive neural networks that take patient information as input. The HMRNN estimates all parameters simultaneously via gradient descent. Using a dataset of Alzheimer's disease patients, we demonstrate how the {\modelacro} can combine an HMM with other predictive neural networks to improve disease forecasting and to offer a novel clinical interpretation compared with a standard HMM trained via expectation-maximization.
\end{abstract}

\begin{IEEEkeywords}
hidden Markov models, neural networks, disease progression
\end{IEEEkeywords}

\section{Introduction}
Hidden Markov models (HMMs; \cite{hmm_seminal}) are commonly used for modeling disease progression, because they capture complex and noisy clinical measurements as originating from a smaller set of latent health states. 
When fit to patient data, HMMs yield \textit{state transition probabilities} (the probabilities of patients transitioning between latent health states) and \textit{emission probabilities} (the probabilities of observing patients' true health states), both of which can have useful clinical interpretations. State transition probabilities describe the dynamics of disease progression, while emission probabilities describe the accuracy of clinical tests and measurements. Because of their intuitive parameter interpretations and flexibility, HMMs have been used to model sepsis \cite{hmm_sepsis}, Alzheimer's progression \cite{continuous_hmm}, 
and patient response to blood anticoagulants \cite{nemani_paper}.

Researchers may wish to use patient-level covariates to improve the fit of HMM parameter solutions \cite{tumor_cytogenetics}, or to integrate HMMs directly with treatment planning algorithms \cite{nemani_paper}. 
Either modification requires incorporating additional parameters into the HMM, which is typically intractable with expectation-maximization algorithms. Incorporating covariates or additional treatment planning models therefore requires multiple estimation steps (e.g., \cite{tumor_cytogenetics}) changes to HMM parameter interpretation (e.g., \cite{nemani_paper}), or Bayesian estimation, which involves joint prior distributions over all parameters and can suffer from poor convergence in complex models \cite{ryden2008bayesian}.

We present neural networks as a valuable alternative for implementing and solving HMMs for disease progression modeling. Neural networks' substantial modularity allows them to easily incorporate additional input variables (e.g., patient-level covariates) or predictive models and simultaneously estimate all parameters \cite{caelli1999modularity}.
We therefore introduce
Hidden Markov Recurrent Neural Networks ({\modelacro}s) - recurrent neural networks (RNNs) that mimic the computation of hidden Markov models while allowing for substantial modularity with other predictive networks. 

When trained on state observation data only (i.e., the same data used to train HMMs), the {\modelacro} has the same parameter set and likelihood function as a traditional HMM. Yet the HMRNN's neural network structure allows it to incorporate additional patient data; in this case, the HMRNN's likelihood function differs from that of a traditional HMM, but it still yields interpretable state transition and emission parameters (unlike other classes of neural networks). Thus, HMRNNs are a type of RNN with a specific structure that allows their parameters to be interpreted as HMM paramters, and HMRNNs can in turn be reduced to standard HMMs. In this way, {\modelacro}s balance the interpretability of HMMs with the flexibility of neural networks (as shown in Figure \ref{fig:fig2}).


\begin{figure}
    \centering
    \includegraphics[scale=.8]{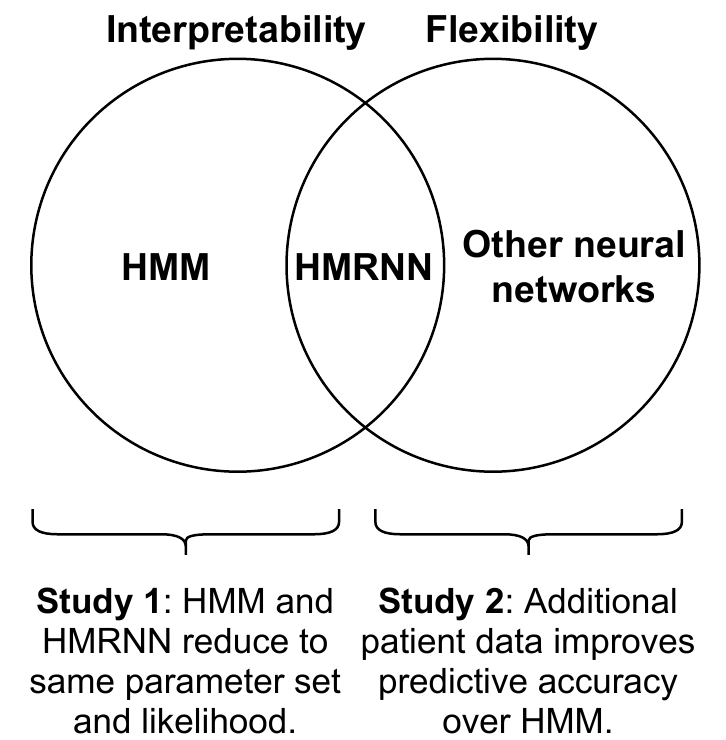}
    \caption{Conceptual overview of HMRNN. `Interpretability' refers to
    model's ability to produce interpretable state transition and emission parameters. `Flexibility' refers to model's ability to incorporate additional data sources in its predictions. When trained on the same data, HMRNNs and HMMs optimize the same likelihood functions and yield the same interpretable parameter solutions (demonstrated in Study 1). The HMRNN can also combine an HMM with an additional predictive neural network, allowing it to use additional patient data not available to standard HMMs (demonstrated in Study 2). 
    }
    \label{fig:fig2}
\end{figure}

Our primary contributions are as follows: (1) We prove how recurrent neural networks (RNNs) can be formulated to optimize the same likelihood function as HMMs, with parameters that can be interpreted as HMM parameters (sections \ref{sec:methods} and \ref{sec:study1}), and (2) we demonstrate the \modelacro's utility for disease progression modeling, by combining it with other predictive neural networks to improve forecasting and offer unique parameter interpretations not afforded by simple HMMs (section \ref{sec:study2}).

\section{Related work}\label{sec:related_research}
A few studies in the speech recognition literature model HMMs with neural networks \cite{wessels_hmm_rnn,alpha_nets}; these implementations require HMM pre-training \cite{wessels_hmm_rnn} or minimize the mutual information criterion \cite{alpha_nets},
and they are not commonly used outside the speech recognition domain.
These works also  
present only theoretical justification, with no empirical comparisons with expectation-maximization algorithms.

A limited number of healthcare studies have also explored connections between neural networks and Markov models. \cite{nemani_paper} employs a recurrent neural network to approximate latent health states underlying patients' ICU measurements. \cite{hmm_nn_surgery} compares HMM and neural network effectiveness in training a robotic surgery assistant, while 
\cite{vae_rl_paper} proposes a generative neural network for modeling ICU patient health based on
HMMs.
These studies differ from our approach of directly formulating HMMs as neural networks, which maintains the interpretability of HMMs while allowing for joint estimation of the HMM with other predictive models.


\section{Methods}\label{sec:methods}

In this section, we briefly review HMM preliminaries, formally define the {\modelacro}, and prove that it optimizes the same likelihood function as a corresponding HMM.

\subsection{HMM preliminaries}\label{sec:preliminaries}
Formally, an HMM models a system over a given time horizon $T$, where the system occupies a hidden state $x_{t} \in S=\{1,\dots,k\}$ at any given time point $t\in \{0,1,\dots,T\}$; that is, $x_{t}=i$ indicates that the system is in the $i$-th state at time $t$. For any state $x_{t} \in S$ and any time point $t \in \{0,1, \dots,T\}$, the system emits an observation according to an emission distribution that is uniquely defined for each state.
We consider the case of categorical emission distributions, which are commonly used in healthcare (e.g., \cite{HMM_breastcancer, hmm_sepsis}).
These systems emit a discrete-valued observation $y_{t} \in O$ at each time $t$,
where $O=\{1,\dots,c\}$.

Thus, an HMM is uniquely defined by a $k$-length initial probability vector $\bm{\pi}$, $k \times k$ transition matrix $\bm{P}$, and $k \times c$ emission matrix $\bm{\Psi}$. Entry $i$ in the vector $\bm{\pi}$ is the probability of starting in state $i$, row $i$ in the matrix $\bm{P}$ is the state transition probability distribution from state $i$, and row $i$ of the matrix $\bm{\Psi}$ is the emission distribution from state $i$. We also define $\diag(\bm{\Psi}_{i})$ as a $k \times k$ diagonal matrix with the $i$-th column of $\bm{\Psi}$ as its entries (i.e., the probabilities of observation $i$ from each of the $k$ states). We define the likelihood of an observation sequence $\bm{y}$ in terms of $\alpha_{t}(i)$, the probability of being in state $i$ at time $t$ \textit{and} having observed $\{y_{0},...,y_{t}\}$. We denote $\bm{\alpha}_{t}$ as the (row) vector of all $\alpha_{t}(i)$ for $i \in S$, with
\begin{equation}
    \bm{\alpha}_{t}=\bm{\pi}^\top \cdot \diag(\bm{\Psi}_{y_{0}}) \cdot (\prod_{i=1}^{t} \bm{P} \cdot \diag(\bm{\Psi}_{y_i}))
\label{eqn:alpha}
\end{equation}
for $t \in \{1,...,T\}$, with $\bm{\alpha}_{0}=\bm{\pi}^\top \cdot \diag(\bm{\Psi}_{y_{0}})$. The likelihood of a sequence $\bm{y}$ is thus given by $\textrm{Pr}(\bm{y})=\bm{\alpha}_{T} \cdot \textbf{1}_{k \times 1}$.


\subsection{{\modelacro} definition}
\label{sec:HMRNN}
An {\modelacro} is a recurrent neural network whose parameters directly correspond to the initial state, transition, and emission probabilities of an HMM. As such, training an HMRNN optimizes the joint log-likelihood of the  $N$ $T$-length observation sequences given these parameters.

\begin{definition}
An {\modelacro} is a recurrent neural network with parameters $\bm{\pi}$ (a $k$-length vector whose entries sum to 1), $\bm{P}$ (a $k \times k$ matrix whose rows sum to one), and $\bm{\Psi}$ (a $k \times c$ matrix whose rows sum to one). It receives $T+1$ input matrices of size $N \times c$, denoted by $\bm{Y}_{t}$ for $t \in \{0,1,\dots,T\}$, where the $n$-th row of matrix $\bm{Y}_{t}$ is a one-hot encoded vector of observation $y_{t}^{(n)}$ for sequence $n \in \{1,\dots,N\}$. The {\modelacro} consists of an inner block of hidden layers that is looped $T+1$ times (for $t \in \{0,1,\dots,T\}$), with each loop containing hidden layers $\bm{h}_{1}^{(t)}$, $\bm{h}_{2}^{(t)}$, and $\bm{h}_{3}^{(t)}$, and a $c$-length input layer $\bm{h}_{y}^{(t)}$ through which the input matrix $\bm{Y}_{t}$ enters the model. The {\modelacro} has a single output unit $o^{(T)}$ whose value is the joint negative log-likelihood of the $N$ observation sequences
under an HMM with parameters $\bm{\pi}$, $\bm{P}$, and $\bm{\Psi}$; the summed value of $o^{(T)}$ across all $N$
observation sequences
is the loss function (minimized via neural network optimization, such as gradient descent).

Layers $\bm{h}_{1}^{(t)}$, $\bm{h}_{2}^{(t)}$, $\bm{h}_{3}^{(t)}$, and $o^{(T)}$ are defined in the following equations. Note that the block matrix in equation (\ref{eqn:GHMNN_2}) is a $c \times (kc)$ block matrix of $c$ $\bm{1}_{1 \times k}$ vectors, arranged diagonally, while the block matrix in equation (\ref{eqn:GHMNN_3}) is a $(kc) \times k$ row-wise concatenation of $c$ $k \times k$ identity matrices.

\begin{align}
\label{eqn:GHMNN_1}
  \bm{h}_{1}^{(t)} & = \begin{cases}
    \bm{\pi}^\top, & t=0,\\
    \bm{h}_{3}^{(t-1)} \bm{P}, & t>0.
  \end{cases}
  \\
\label{eqn:GHMNN_2}
  \bm{h}_{2}^{(t)} & = \!\begin{aligned}[t] & \relu
  \Big(\bm{h}_{1}^{(t)} \begin{bmatrix} \diag(\bm{\Psi}_{1}) \dots \diag(\bm{\Psi}_{c}) \\\end{bmatrix}+
\\
  & \bm{Y}_{t}\begin{bmatrix} \bm{1}_{1 \times k} && \dots && \bm{0}_{1 \times k} \\
  \dots && \dots && \dots \\
  \bm{0}_{1 \times k} && \dots && \bm{1}_{1 \times k} \end{bmatrix} - \bm{1}_{n \times (kc)}
  \Big)
  \end{aligned}
  \\
\label{eqn:GHMNN_3}
  \bm{h}_{3}^{(t)} & = \bm{h}_{2}^{(t)}\begin{bmatrix} \bm{I}_{k} & \dots &  \bm{I}_{k}\\
\end{bmatrix}^{\top}
  \\
\label{eqn:GHMNN_4}
  o^{(t)} & = -\log(\bm{h}_{3}^{(T)}\mathbf{1}_{k \times 1}).
\end{align}

\end{definition}
\begin{figure*}[t]
    \centering
    \includegraphics[scale=.35]{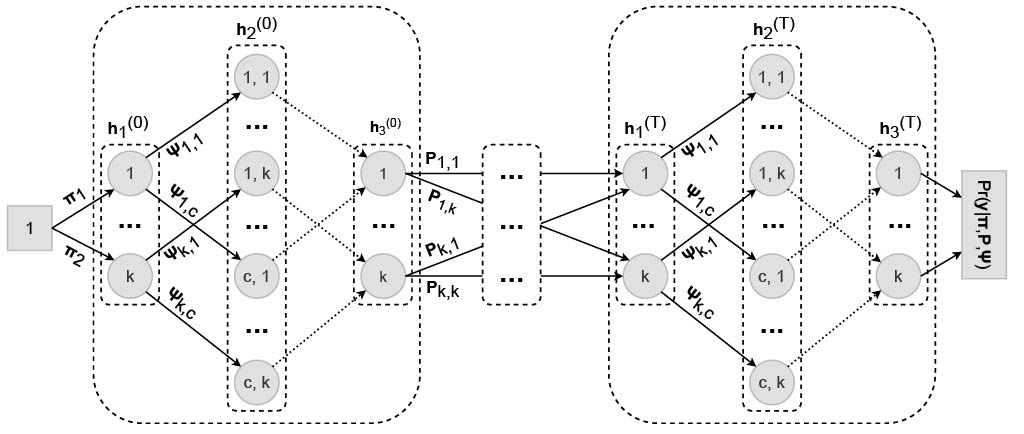}
    \caption{Structure of the hidden Markov recurrent neural network ({\modelacro}). Solid lines indicate learned weights that correspond to HMM parameters; dotted lines indicate weights fixed to 1. The inner block initializes with the initial state probabilities then mimics multiplication by $\diag(\bm{\Psi}_{y_{t}})$; connections between blocks mimic multiplication by $\bm{P}$.}
    \label{fig:GHMNN_diagram}
\end{figure*}
Fig.\ \ref{fig:GHMNN_diagram} outlines the structure of the {\modelacro}. Note that layer $\bm{h}_{3}^{(t)}$ is equivalent to $\bm{\alpha}_{t}$, the probability of being in each hidden state given $\{y_{0},...,y_{t}\}$. Also note that, for long sequences, underflow can be addressed by normalizing layer $\bm{h}_{3}^{(t)}$ to sum to 1 at each time point, then simply subtracting the logarithm of the normalization term (i.e., the log-sum of the activations) from the output $o^{(T)}$.

\subsection{Proof of HMM/{\modelacro} equivalence}\label{sec:proof}
We now formally establish  that the {\modelacro}'s output unit, $o^{(T)}$, is the negative log-likelihood of an observation sequence under an HMM with parameters $\bm{\pi}$, $\bm{P}$, and $\bm{\Psi}$. We prove this for the case of $N=1$ and drop notational dependence on $n$ (i.e., we write $y_{t}^{(1)}$ as $y_{t}$), though extension to $N>1$ is trivial since the log-likelihood of multiple independent sequences is the sum of their individual log-likelihoods. We first rely on the following lemma.

\begin{lemma}\label{lemma1}
If all units in $\bm{h}_{1}^{(t)}(j)$ are between 0 and 1 (inclusive), then $\bm{h}_{3}^{(t)}=\bm{h}_{1}^{(t)} \diag(\bm{\Psi}_{y_{t}})$.
\end{lemma}

\begin{proof}
Let $\bm{h}_{1}^{(t)}(j)$ and $\bm{h}_{3}^{(t)}(j)$ represent the $j$th units of layer $\bm{h}_{1}^{(t)}$ and $\bm{h}_{3}^{(t)}$, respectively, and recall that 
$\bm{h}_{2}^{(t)}$ contains $k \times c$ units, which we index with a tuple $(l,m)$ for $l \in \{1,\dots,c\}$ and $m \in \{1,\dots,k\}$. According to equation (\ref{eqn:GHMNN_2}), the connection between units $\bm{h}_{1}^{(t)}(j)$ and $\bm{h}_{2}^{(t)}(l,m)$ is $\bm{\Psi}_{j,l}$ when $j=m$,  and 0 otherwise. Also recall that matrix $\bm{Y}_{t}$ enters the model through a $c$-length input layer that we denote $\bm{h}_{y}^{(t)}$. According to equation (\ref{eqn:GHMNN_3}),
the connection between unit $\bm{h}_{y}^{(t)}(j)$ and unit $\bm{h}_{2}^{(t)}(l,m)$ is 1 when $j=l$, and 0 otherwise. Thus, unit $\bm{h}_{2}^{(t)}(l,m)$ depends only on $\bm{\Psi}_{m,l}$, $\bm{h}_{1}^{(t)}(m)$, and $\bm{h}_{y}^{(t)}(l)$. Lastly, a bias of $-1$ is added to all units in $\bm{h}_{2}^{(t)}$, which is then subject to a ReLu activation, resulting in the following expression for each unit in $\bm{h}_{2}^{(t)}$:
\begin{equation}\label{eqn:layer2}
   \bm{h}_{2}^{(t)}(l,m)=\relu(\bm{\Psi}_{m,l}\cdot \bm{h}_{1}^{(t)}(m)+\bm{h}_{y}^{(t)}(l)-1).
\end{equation}
\looseness-1 Because $\bm{h}_{y}^{(t)}(l)$ is 1 when $y_{t}=l$, and equals  0 otherwise, then
if all units in $\bm{h}_{1}^{(t)}$ are between 0 and 1, this implies $\bm{h}_{2}^{(t)}(l,m)=\bm{\Psi}_{m,l}\cdot \bm{h}_{1}^{(t)}(m)$ when $j=y_{t}$ and $\bm{h}_{2}^{(t)}(l,m)=0$ otherwise. According to equation (\ref{eqn:GHMNN_4}), the connection between $\bm{h}_{2}^{(t)}(l,m)$ and $\bm{h}_{3}^{(t)}(j)$ is 1 if $j=m$, and 0 otherwise. Hence,
\begin{equation}
    \bm{h}_{3}^{(t)}(j)=\sum_{j=0}^{c} \bm{h}_{2}^{(t)}(l,j)=\bm{\Psi}_{j,y_{t}}\cdot \bm{h}_{1}^{(t)}(j).
\end{equation}
Thus, $\bm{h}_{3}^{(t)}=\bm{h}_{1}^{(t)} \diag(\bm{\Psi}_{y_{t}})$.
\end{proof}
\begin{theorem}
\label{maintheorem}
An {\modelacro} with parameters $\bm{\pi}$ ($1 \times k$ stochastic vector), $\bm{P}$ ($k \times k$ stochastic matrix), and $\bm{\Psi}$ ($k \times c$ stochastic matrix), and with layers defined as in equations (\ref{eqn:GHMNN_1}-\ref{eqn:GHMNN_4}), produces output neuron $o^{(T)}$ 
whose value is the negative log-likelihood of a corresponding HMM.
\end{theorem}
\begin{proof}
Note that, based on Lemma \ref{lemma1} and equation (\ref{eqn:GHMNN_1}), $\bm{h}_{3}^{(t)}=\bm{h}_{3}^{(t-1)} \cdot \bm{P} \cdot \diag(\bm{\Psi}_{y_{t}})$ for $t \in \{1,...,T\}$, assuming that $\bm{h}_{1}^{(t)}(j) \in [0,1]$ for $j \in \{1,..,k\}$. Since $\bm{\alpha}_{t}=\bm{\alpha}_{t-1} \cdot \bm{P} \cdot \diag(\bm{\Psi}_{y_t})$, then if $\bm{h}_{3}^{(t-1)}=\bm{\alpha}_{t-1}$, then
$\bm{h}_{1}^{(t)}(j) \in [0,1]$ for $j \in \{1,..,k\}$ and therefore $\bm{h}_{3}^{(t)}=\bm{\alpha}_{t}$. We show the initial condition that $\bm{h}_{3}^{(0)}=\bm{\alpha}_{0}$, since $\bm{h}_{1}^{(0)}=\bm{\pi}^\top$ implies that $\bm{h}_{3}^{(0)}=\bm{\pi}^\top \cdot \diag(\bm{\Psi}_{y_{0}})=\bm{\alpha}_{0}$. Therefore, by induction, $\bm{h}_{3}^{(T)}=\bm{\alpha}_{T}$, and $o^{(T)}=-\log(\bm{\alpha}_{T} \cdot \bm{1}_{k \times1})$, which is the logarithm of the HMM likelihood based on equation (\ref{eqn:alpha}).
\end{proof}
\section{Experiments and Results}\label{sec:results}
In study 1, we demonstrate that when an HMRNN is trained on the same data as an HMM, the two models yield statistically similar parameter estimates.
More specifically, study 1 provides a simulation-based validation of Theorem \ref{maintheorem},
according to which an HMRNN and an HMM trained on the same data share the same likelihood.

In study 2, we use disease progression data to demonstrate that an HMRNN combining an HMM with an additional predictive neural network attains better predictive performance over its HMM component, while still yielding an interpretable parameter solution. Theorem \ref{maintheorem} does not apply to study 2, since the inclusion of an additional predictive neural network to the HMRNN yields a likelihood function different from the likelihood of the consituent HMM.

\subsection{Study 1: HMRNN Reduces to an HMM}\label{sec:study1}
We demonstrate that an {\modelacro} trained via gradient descent yields statistically similar solutions to Baum-Welch.  We show this with synthetically-generated observations sequences for which the true HMM parameters are known.

We simulate systems with state spaces $S={1,2,...,k}$ that begin in state $1$, using $k=5$, $10$, or $20$ states. These state sizes are consistent with disease progression HMMs, which often involve less than 10 states \cite{tumor_cytogenetics,jackson_multistate_models,sukkar_hmm}. We assume that each state `corresponds' to one observation, implying the same number of states and observations ($c=k$). The probability of correctly observing a state ($P(y_{t}=x_{t})$) is $\psi_{ii}$, which is the diagonal of $\bm{\Psi}$ and is the same for all states. We simulate systems with $\psi_{ii}=0.6, 0.75$, and $0.9$.

We test three variants of the transition probability matrix $\bm{P}$. Each is defined by their same-state transition probability $p_{ii}$, which is the same for all states. For all $\bm{P}$ the probability of transitioning to higher states increases with state membership; this is known as `increasing failure rate' and is a common property for Markov processes. As $p_{ii}$ decreases, the rows of $\bm{P}$ stochastically increase, i.e., lower values of $p_{ii}$ imply a greater chance of moving to higher states. We use values of $p_{ii}=0.4, 0.6$, and $0.8$, for 27 total simulations ($k=\{5,10,20\} \times \bm{\Psi}_{ii}=\{0.6,0.75,0.9\} \times p_{ii}=\{0.4,0.6,0.8\}$).

For each of the 27 simulations, we generate 100 trajectories of length $T=60$; this time horizon might practically represent one hour of data collected each minute or two months of data collected each day. Initial state probabilities are fixed at $1$ for state $1$ and $0$ otherwise. Transition parameters are initialized based on the observed number of transitions in each dataset, using each observation as a proxy for its corresponding state. Since transition probabilities are initialized assuming no observation error, the emission matrices are correspondingly initialized using $\psi_{ii}=0.95$ (with the remaining 0.05 distributed evenly across all other states). For Baum-Welch and HMRNN, training ceased when all parameters ceased to change by more than 0.001. For each simulation, we compare Baum Welch's and the HMRNN's average Wasserstein distance between the rows of the estimated and ground truth $\bm{P}$ and $\bm{\Psi}$ matrices. This serves as a measure of each method's ability to recover the true data-generating parameters. We also compare the Baum-Welch and HMRNN solutions' log-likelihoods using a separate hold-out set of 100 trajectories.

Across all simulations, the average Wasserstein distance between the rows of the true and estimated transition matrices was 0.191 for Baum-Welch and 0.178 for HMRNN (paired $t$-test $p$-value of 0.483). For the emission matrices, these distances were 0.160 for Baum-Welch and 0.137 for HMRNN (paired $t$-test $p$-value of 0.262). This suggests that Baum-Welch and the HMRNN recovered the ground truth parameters with statistically similar degrees of accuracy. This can be seen in Figure \ref{fig:p_recovery}, which presents the average estimated values of $p_{ii}$ and $\psi_{ii}$ under each model. Both models' estimated $p_{ii}$ values are, on average, within 0.05 of the ground truth values, while they tended to estimate $\psi_{ii}$ values of around 0.8 regardless of the true $\psi_{ii}$. Note that, while Baum-Welch was slightly more accurate at estimating $p_{ii}$ and $\psi_{ii}$, the overall distance between the ground truth and estimated parameters did not significantly differ between Baum-Welch and the HMRNN.

For each simulation, we also compute the log-likelihood of a held-out set of 100 sequences under the Baum-Welch and HMRNN parameters, as a measure of model fit. The average holdout log-likelihoods under the ground truth, Baum-Welch, and HMRNN parameters are -9250.53, -9296.03, and -9303.27, respectively (paired $t$-test $p$-value for Baum-Welch/HMRNN difference of 0.440). Thus, Baum-Welch and HMRNN yield similar model fit on held-out data.

\begin{figure}
    \centering
    \includegraphics[scale=0.3]{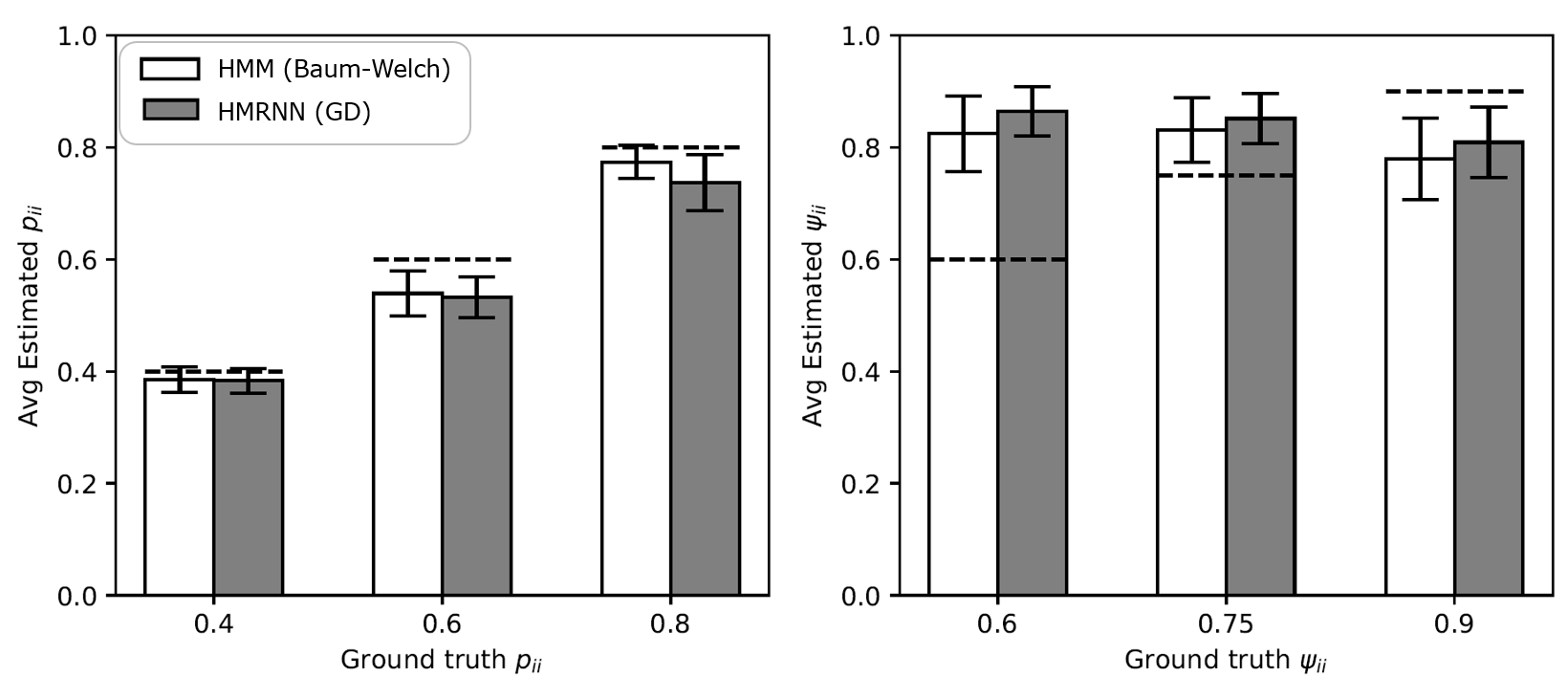}
    \caption{Results from Study 1. GD=Gradient Descent. Estimated $p_{ii}$ (left) and $\psi_{ii}$ (right) under Baum-Welch and HMRNN, shown by ground truth parameter value. Results for each column are averaged across 9 simulations. Dashed lines indicate ground truth $p_{ii}$ (left) and $\psi_{ii}$ (right) values, and error bars indicate 95\% confidence intervals (but do not represent tests for significant differences). Baum-Welch and the HMRNN produce near-identical parameter solutions according to the Wasserstein distance metric.}
    \label{fig:p_recovery}
\end{figure}

\subsection{Study 2: HMRNN Improves Predictive Accuracy over HMM}
\label{sec:study2}
We demonstrate how combining an {\modelacro} with other predictive neural networks improves predictive accuracy and offers novel clinical interpretations over a standard HMM, using an Alzheimer's disease case study. Recall that, by incorporating an additional predictive neural network into an {\modelacro}, its likelihood function differs from that of a traditional HMM but it still produces interpretable state transition and emission parameters. We test our {\modelacro} on clinical data from $n=426$ patients with mild cognitive impairment (MCI), collected over the course of three ($n=91$), four ($n=106$), or five ($n=229$) consecutive annual clinical visits \cite{adni}. Given MCI patients' heightened risk of Alzheimer's, modeling their symptom progression is of considerable clinical interest. We analyze patients' overall cognitive functioning based on the Mini Mental Status Exam (MMSE; \cite{MMSEref}).

MMSE scores range from 0 to 30, with score categories for `no cognitive impairment' (scores of 27-30), `borderline cognitive impairment' (24-26), and 'mild cognitive impairment' (17-23) \cite{MMSE_cutoffs_review}. Scores below 17 were infrequent (1.2\%) and were treated as scores of 17 
for analysis. We use a 3-state latent space $S=\{0,1,2\}$, with $x_{t}=0$ representing `no cognitive impairment,' $x_{t}=1$ representing `borderline cognitive impairment,' and $x_{t}=2$ representing `mild cognitive impairment.' The observation space is $O=\{0,1,2\}$, using $y_{t}=0$ for scores of $27-30$, $y_{t}=1$ for scores of $24-26$, and $y_{t}=2$ for scores of $17-23$. This HMM therefore allows for the possibility of measurement error, i.e., that patients' observed score category $y_{t}$ may not correspond to their true diagnostic classification $x_{t}$.

To showcase the benefits of the {\modelacro}'s modularity, we 
augment it with two predictive neural networks. First, we predict patient-specific initial state probabilities based on gender, age, degree of temporal lobe atrophy, 
and amyloid-beta 42 levels (A$\beta$42, a relevant Alzheimer's biomarker \cite{abref_2}), using a single-layer neural network with a softmax activation. Second, at each time point, the probability of being in the most impaired state,  
$\bm{h}_{t}^{(1)}(2)$, 
is used to predict concurrent scores on the Clinical Dementia Rating (CDR, \cite{CDRref}), a global assessment of dementia severity, allowing another relevant clinical metric to inform paramter estimation. We use a single connection and sigmoid activation to predict patients' probability of receiving a CDR score above 0.5 (corresponding to `mild dementia'). The HMRNN is trained via gradient descent to minimize $o^{(T)}$ from equation (\ref{eqn:GHMNN_4}), plus the predicted negative log-likelihoods of patients' CDR scores. Figure \ref{fig:mmse_hmrnn} visualizes the structure of this augmented HMRNN.

\begin{figure}[t]
    \centering
    \includegraphics[scale=.15]{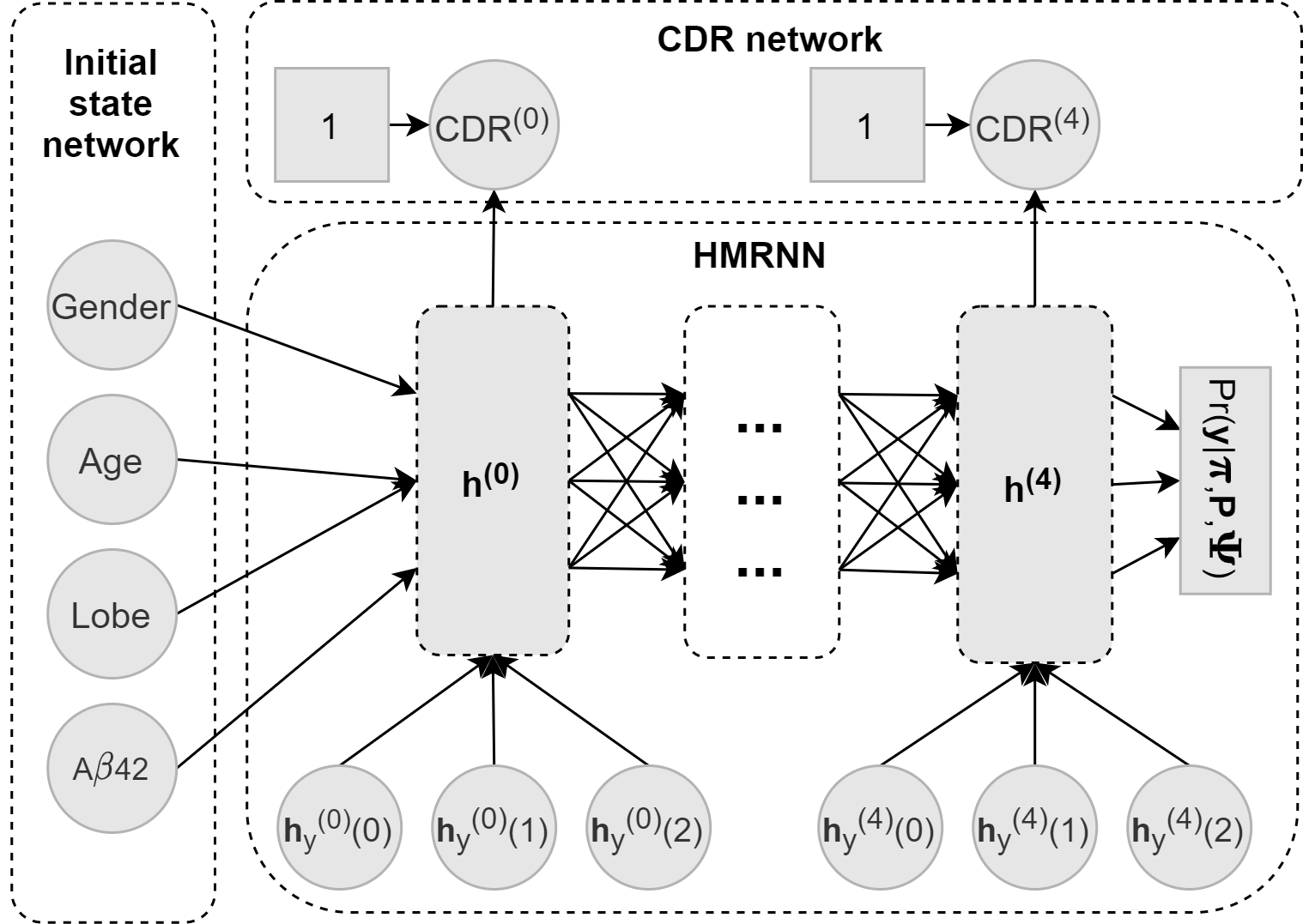}
    \caption{Augmented HMRNN for Alzheimer's case study. CDR$^{(t)}$ refers to predicted CDR classification (above or below 0.5) at time $t \in \{0,1,2,3,4\}$. `Lobe' refers to measure of temporal lobe atrophy. Units $\bm{h}_{y}^{(t)}$ are a one-hot encoded representation of the MMSE score category at time $t$.}
    \label{fig:mmse_hmrnn}
\end{figure}

We compare the {\modelacro} to a standard HMM without these neural network augmentations, trained using Baum-Welch, an expectation-maximization algorithm \cite{hmm_seminal}. We assess parameter solutions' ability to predict patients' final MMSE score categories from their initial score categories, using 10-fold cross-validation. We evaluate performance using weighted log-loss $L$, 
i.e., the average log-probability placed on each final MMSE score category. This metric accounts for class imbalance and rewards models' confidence in their predictions, an important component of medical decision support \cite{bussone_trust}. We also report $\bar{p}$, the average probability placed on patients' final MMSE scores (computed directly from $L$). We train all models using a relative log-likelihood tolerance of $0.001\%$. Runtimes for Baum-Welch and the HMRNN are 2.89 seconds and 15.24 seconds, respectively.

Model results appear in Table \ref{tab:mmse_parameters}. Note that the {\modelacro}'s weighted log-loss $L$ is significantly lower than Baum-Welch's (paired $t$-test $\mbox{p-value}=2.396\times10^{-6}$), implying greater predictive performance. This is supported by Figure \ref{fig:mmse_plot}, which shows $\bar{p}$, the average probability placed on patients' final MMSE scores by score category. Note that error bars represent marginal sampling error and do not represent statistical comparisons between Baum-Welch and HMRNN.
The {\modelacro} also yields lower transition probabilities and lower estimated diagnostic accuracy for the MMSE (i.e., lower diagonal values of $\bm{\Psi}$) than Baum-Welch, For instance, the baseline HMM estimates at least an 80\% chance of correctly identifying borderline and mild cognitive impairment ($\bm{\Psi}_{22} = 0.819$ and $\bm{\Psi}_{33} = 0.836$). These probabilities are (respectively) only 54.8\% and 68.7\% under the {\modelacro},
suggesting that score changes are more likely attributable to testing error as opposed to true state changes.

\begin{table}[t]
\caption{Results from Alzheimer's disease case study. $\bm{\pi}$ is initial state distribution, $\bm{P}$ is state transition matrix, $\bm{\Psi}$ is emission distribution matrix, $L$ is weighted log-loss, and $\bar{p}$ is average probability placed on ground truth score categories.}
    \centering
    \begin{tabular}{|l|c|c|}
         \hline
         & Baum-Welch & {\modelacro} \\
         \hline
         $\bm{\pi}$ & $ \begin{array}{ccc} 0.727 & 0.271 & 0.002 \end{array}$ & $ \begin{array}{ccc} 0.667 & 0.333 & 0.000 \end{array}$ \\
         \hline
         $\bm{P}$ & $\begin{array}{ccc} 0.898 & 0.080 & 0.022 \\ 0.059 & 0.630 & 0.311 \\ 0.000 & 0.016 & 0.984  \end{array}$ & $\begin{array}{ccc} 0.970 & 0.028 & 0.002 \\ 0.006 & 0.667 & 0.327 \\ 0.000 & 0.003 & 0.997  \end{array}$\\
         \hline
         $\bm{\Psi}$ & $\begin{array}{ccc} 0.939 & 0.060 & 0.001  \\ 0.175 & 0.819 & 0.006 \\ 0.004 & 0.160 & 0.836 \end{array}$ & $\begin{array}{ccc} 0.930 & 0.067 & 0.003 \\ 0.449 & 0.548 & 0.003 \\ 0.005 & 0.308 & 0.687  \end{array}$ \\
         \hline
         $L$ & -0.992 & -0.884 \\
         \hline
         $\bar{p}$ & 0.371 & 0.413 \\
         \hline

    \end{tabular}
    \label{tab:mmse_parameters}
\end{table}
\begin{figure}
    \centering
    \includegraphics[scale=0.3]{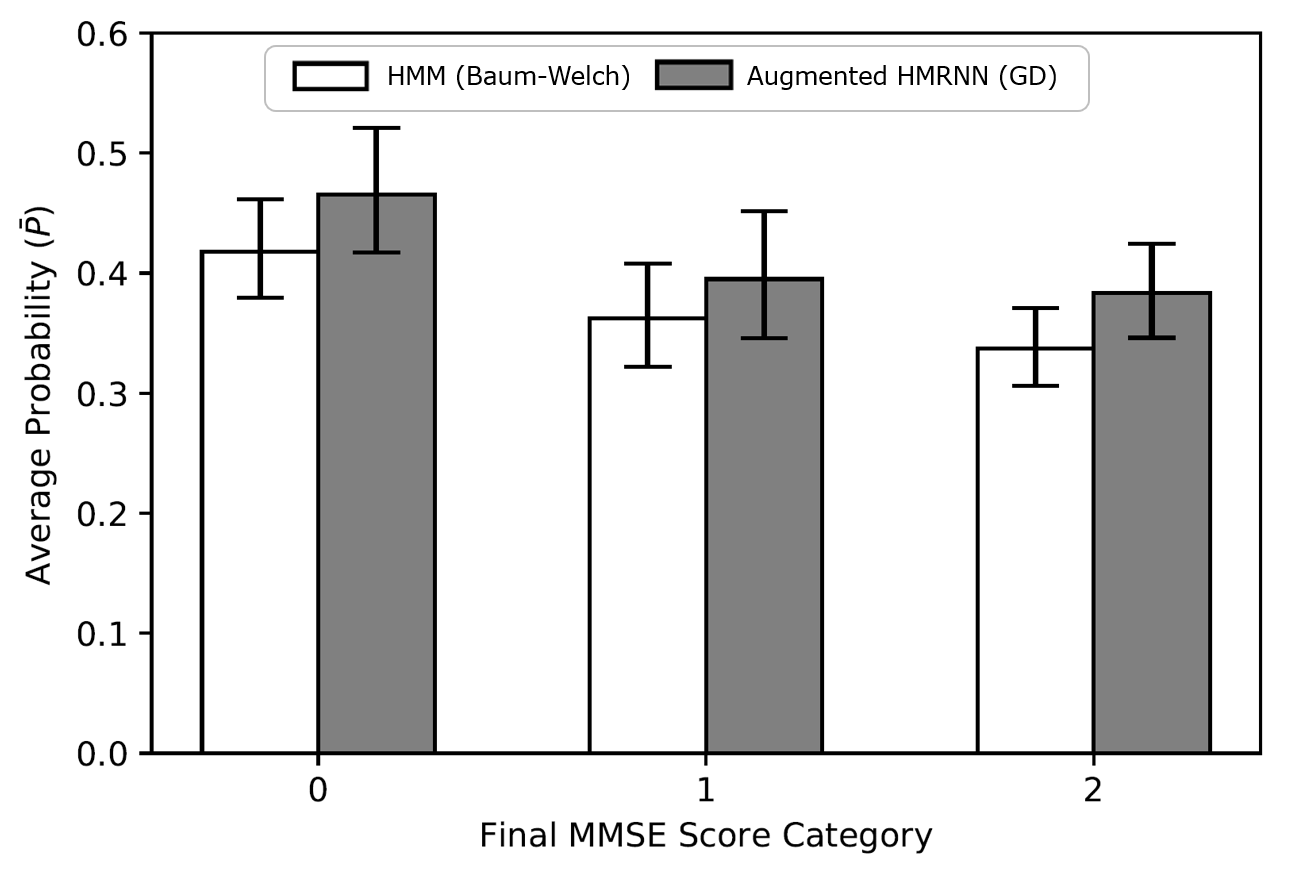}
    \caption{Results from study 2. GD=Gradient Descent. Plot shows average probability placed on final MMSE scores, by score category. Recall that the HMRNN's average performance significantly outperforms Baum-Welch (paired $t$-test $p$-value$=2.396\times10^{-6}$). As see in the Figure, this effect is consistent across score categories. Error bars indicate 95\% confidence intervals, and do not represent tests for significant differences.}
    \label{fig:mmse_plot}
\end{figure}

\section{Discussion}\label{sec:discussion}
We outline a flexible approach for HMM estimation using neural networks. The {\modelacro} produces statistically similar solutions to HMMs when trained on the same data. It can also combine HMMs with other neural networks to improve disease progression forecasting when additional patient data is available. 
In our Alzheimer's disease experiment (study 2), augmenting an {\modelacro} with two predictive networks improves forecasting performance compared with a standard HMM trained with Baum-Welch. The {\modelacro} also yields a clinically distinct parameter interpretation, predicting poor diagnostic accuracy for the MMSE's `borderline' and `mild' impairment categories. This suggests that fewer diagnostic categories might improve MMSE utility, which aligns with existing research \cite{MMSE_cutoffs_review} and suggests the {\modelacro} might be used to improve the clinical utility of HMM parameter solutions. We also make a novel theoretical contribution by formulating discrete-observation HMMs as a special case of RNNs and proving coincidence of their likelihood functions.


Future work might formally assess {\modelacro} time complexity. Yet since data sequences in healthcare are often shorter than in other domains that employ HMMs (e.g., speech analysis), runtimes will likely be reasonable for many healthcare datasets. Future work might explore the {\modelacro} in other healthcare applications besides disease progression. Lastly, while we the address the case of discrete-state, discrete-time HMMs, neural networks might also be used to implement more complex HMM structures. For instance, the HMRNN might be extended to continuous-time HMMs, for which parameter estimation is quite difficult. The HMRNN also might be extended to partially-observable Markov decision processes (POMDPs), in which latent state transitions are affected by actions taken at each time point. Since actions are a form of time-varying covariates, the HMRNN structure would easily allow for parameter estimation when action data are available.

\section*{Acknowledgment}
This research is partially supported by the Joint Directed Research and Development program at Science Alliance, University of Tennessee. Data collection and sharing for this project was funded by the Alzheimer's Disease Neuroimaging Initiative (ADNI) (National Institutes of Health Grant U01 AG024904) and DOD ADNI (Department of Defense award number W81XWH-12-2-0012).

This manuscript has been authored by UT-Battelle, LLC, under contract DE-AC05-00OR22725 with the US Department of Energy (DOE). The US government retains and the publisher, by accepting the article for publication, acknowledges that the US government retains a nonexclusive, paid-up, irrevocable, worldwide license to publish or reproduce the published form of this manuscript, or allow others to do so, for US government purposes. DOE will provide public access to these results of federally sponsored research in accordance with the DOE Public Access Plan (http://energy.gov/down-loads/doe-public-access-plan). This research was sponsored by the Laboratory Directed Research and Development Program of Oak Ridge National Laboratory, managed by UT-Battelle, LLC, for the US Department of Energy under contract DE-AC05-00OR22725.


\end{document}